\documentclass[10pt,twocolumn]{article} 
\usepackage{simpleConference}
\usepackage{times}
\usepackage{graphicx}
\usepackage{amssymb}
\usepackage{url,hyperref}
\usepackage{amsmath}               
\usepackage{amsfonts}              
\usepackage{caption}
\usepackage{amsthm}
\usepackage[titletoc]{appendix}
\usepackage{chngcntr}

\newtheorem{theorem}{Theorem}[section]
\newtheorem{lemma}[theorem]{Lemma}
\newtheorem{definition}[theorem]{Definition}
\newtheorem{corollary}[theorem]{Corollary}
\numberwithin{equation}{section}

\begin{document}

\title{\Large Modularity Component Analysis versus Principal Component Analysis}
\author{Hansi Jiang\\
\\
North Carolina State University\\
Raleigh, NC, 27695, USA\\
\\
hjiang6@ncsu.edu\\
\and
Carl Meyer\\
\\
North Carolina State University\\
Raleigh, NC, 27695, USA\\
\\
meyer@ncsu.edu\\
}

\maketitle

\begin{abstract} \small\baselineskip=9pt 
In this paper the exact linear relation between the leading eigenvectors of the modularity matrix and the singular vectors of an uncentered data matrix is developed. Based on this analysis the concept of a modularity component is defined, and its properties are developed. It is shown that modularity component analysis can be used to cluster data similar to how traditional principal component analysis is used  except that modularity component analysis does not require data centering. 
\end{abstract}



\section{Introduction}
The purpose of this paper is to present a development of modularity components that are analogous to principal value components \cite{jolliffe2002principal}. It will be shown that modularity components have characteristics that are similar to those of principal value components in the sense that modularity components provide for data analysis in much the same manner as do principal value components. In particular, just as in the case of principal value components, modularity components are shown to be mutually orthogonal, and raw data can be projected onto the directions of a number of modularity components to reveal patterns and clusters in the data. However, a drawback of principal component analysis (PCA) is that it generally requires centering or standardizing the data before determining principal components. On the other hand, utilizing modularity components does not require data to be centered to accurately extract important information. Among other things, this means that sparsity in the original data is preserved whereas centering data naturally destroys inherent sparsity.\\
\newline
Moreover, we will complete the comparison of modularity components with principal components by showing that the component that maximizes the modularity function of the \textit{uncentered} data as defined in \cite{newman2006modularity} can replace the principal component that maximizes the variance in the \textit{centered} data. Finally, just as each succeeding principal component has maximal variance with the constraint that it is orthogonal to all previous principal components, each succeeding modularity component has maximal modularity with the constraint that it is orthogonal to all prior modularity components. \\
\newline
Our modularity components are derived from the concept of modularity introduced by Newman and Girvan in \cite{newman2004finding}, and further explained by Newman in \cite{newman2006modularity}.  The modularity partitioning method starts with an adjacency matrix or similarity matrix and aims to partition a graph by maximizing the modularity. Assuming the graph containing $n$ nodes, the modularity is defined by
\begin{equation}\label{eq1}
Q(\mathbf{s})=\frac{1}{4m}\mathbf{s}^T\mathbf{Bs},
\end{equation}
where $m$ is the number of edges in the graph, $\mathbf{B}$ is the modularity matrix defined below, and $\bar{\mathbf{s}}\in\mathbb{R}^n$ is a vector that maximizes $Q$. Since the number of edges in a given graph is constant, the multiplier $1/(4m)$ is often dropped for simplicity, and the modularity becomes
\begin{equation}\label{eq2}
Q(\mathbf{s})=\mathbf{s}^T\mathbf{Bs}.
\end{equation}
The modularity matrix is defined by
\begin{equation}\mathbf{B}=\mathbf{A}-\frac{\mathbf{d}\mathbf{d}^T}{2m},
\end{equation}
where $\mathbf{A}$ is an adjacency matrix or similarity matrix, and $\mathbf{d}=\begin{pmatrix}d_1&d_2&\cdots&d_n\end{pmatrix}^T$ is the vector containing the degrees of the nodes. It is proven in \cite{newman2006modularity} that the eigenvector corresponding to the largest eigenvalue of $\mathbf{B}$ can maximize $Q$. Like the spectral clustering method \cite{von2007tutorial}, the modularity clustering method also uses signs of entries in the dominant eigenvector to partition graphs. \\
\newline
The modularity partitioning algorithm has been widely applied and discussed. For instance, it has been applied to reveal human brain functional networks \cite{meunier2010hierarchical} and ecological networks \cite{fortuna2010nestedness}, and used in image processing \cite{mercovich2011automatic}. Blondel et al. \cite{blondel2008fast} proposed a heuristic that can reveal the community structure for large networks. Rotta and Noack \cite{rotta2011multilevel} compared several heuristics in maximizing modularity. DasGupta and Desai \cite{dasgupta2013complexity} studied the complexity of modularity clustering. The limitations of the modularity maximization technique are discussed in \cite{good2010performance} and \cite{lancichinetti2011limits}. \\
\newline
By the modularity algorithm \cite{newman2006modularity}, a graph is partitioned into two parts, and a hierarchy can be built by iteratively calculating the $\mathbf{B}$ matrices and their dominant eigenvectors. Repetitively partitioning a graph into two subsets may be inefficient and does not utilize information in subdominant eigenvectors. And while there is a connection between graph partitioning and data analysis, they are not strictly equivalent because extracting information from raw data by means of graph partitioning necessarily requires the knowledge or creation of a similarity or adjacency matrix, which in turn can only group nodes. For the purpose of data analysis, it is more desirable to analyze raw data without involving a similarity matrix. Modularity analysis can be executed directly from uncentered raw data $\mathbf{X}_{p\times n}$ ($p$ number of attributes, $n$ number of data points) by redefining the modularity matrix to be 
\begin{equation}\mathbf{B}=\mathbf{X}^T\mathbf{X}-\frac{\mathbf{d}\mathbf{d}^T}{2m},
\end{equation}
but in practice $\mathbf{X}^T\mathbf{X}$ need not be explicitly computed. In addition to using only raw data, this formulation allows the creation of modularity components that are directly analogous to principal value components created from centered data. In what follows, let $\mathbf{A}=\mathbf{X}^T\mathbf{X}$, where the rows of $\mathbf{X}$ may be normalized when different units are involved.\\
\newline
The paper is organized as follows. In Section 2 we give the definition of modularity components. In Section 3 properties of the modularity components are established. Section 4 contains some conclusions.  

\section{Definition of Modularity Components}
In this section we will give the definition of the modularity components. Before doing that we will prove a couple of lemmas about the relation between the eigenvectors of a particular kind of similarity matrices that can be fed in the modularity algorithm and the singular vectors of the data matrix. The lemmas will help us to define the modularity components. Suppose the SVD of the uncentered data matrix $\mathbf{X}$ is $\mathbf{X}=\mathbf{U\Sigma V}^T$ and that there are $k$ nonzero singular values. Then 
\begin{equation}\mathbf{A}=\mathbf{X}^T\mathbf{X}=\mathbf{V\Sigma}^T\mathbf{\Sigma V}^T
\end{equation}
has $k$ positive eigenvalues. From the interlacing theorem mentioned in \cite{bunch1978rank} and \cite{wilkinson1965algebraic}, it is guaranteed that the largest $k-1$ eigenvalues of $\mathbf{B}=\mathbf{A}-\mathbf{d}\mathbf{d}^T/(2m)$ are positive. If the $k$ eigenvalues of $\mathbf{A}$ are simple, then the eigenvectors of $\mathbf{B}$ corresponding to the largest $k-1$ eigenvalues can be written as linear combinations of the eigenvectors of $\mathbf{A}$. The proof of the following lemma can be found in Appendix \ref{app1}.
\begin{lemma}\label{thm3}
Suppose the largest $k-1$ eigenvalues of $\mathbf{B}$ are $\beta_1>\beta_2>\cdots>\beta_{k-1}$ and the nonzero eigenvalues of $\mathbf{A}=\mathbf{X}^T\mathbf{X}$ are $\alpha_1>\alpha_2>\cdots>\alpha_{k}$. Further suppose that for $1\le i\le k-1$ we have $\beta_i\ne\alpha_i$ and $\beta_i\ne\alpha_{i+1}$. Then the eigenvector $\mathbf{b}_i$ of $\mathbf{B}$ can be written by
\begin{equation}\mathbf{b}_i=\sum_{j=1}^k\gamma_{ij}\mathbf{v}_j,
\end{equation}
where
\begin{equation}\gamma_{ij}=\frac{\mathbf{v}_j^T\mathbf{d}}{(\alpha_j-\beta_i)\|\mathbf{d}\|_2}.
\end{equation}
\end{lemma}

The point of this lemma is to realize that the vector $\mathbf{b}_i$ is a linear combination of the $\mathbf{v}_i$. The next lemma gives the linear expression of the vectors $\mathbf{b}_i^T\mathbf{X}^{\dagger}$ in terms of the $\mathbf{u}_i$, where $\mathbf{X}^{\dagger}$ is the Moore-Penrose inverse of $\mathbf{X}$. There are practical cases where our assumptions in Lemma \ref{thm3} hold true, and examples are given in Appendix \ref{app2}.

\begin{lemma}\label{thm4}
With the assumptions in Lemma \ref{thm3}, we have 
\begin{equation}\mathbf{b}_i^T\mathbf{X}^{\dagger}=\sum_{j=1}^k\frac{\gamma_{ij}}{\sigma_j}\mathbf{u}_j^T,
\end{equation}
where $\sigma_j$ is the $j$-th the nonzero singular value of $\mathbf{X}$.
\end{lemma}
\begin{proof}
$$\mathbf{b}_i^T\mathbf{X}^{\dagger}=\bigg(\sum_{j=1}^k\gamma_{ij}\mathbf{v}_j^T\bigg)\mathbf{V\Sigma^{\dagger}}\mathbf{U}^T$$
$$=\begin{pmatrix}
\gamma_{i1} & \gamma_{i2} & \cdots &\gamma_{ik} & 0 & \cdots & 0
\end{pmatrix}_{1\times n}\mathbf{\Sigma^{\dagger} U}^T$$
$$=\begin{pmatrix}
\frac{\gamma_{i1}}{\sigma_1} & \frac{\gamma_{i2}}{\sigma_2} & \cdots & \frac{\gamma_{ik}}{\sigma_k} & 0 & \cdots & 0
\end{pmatrix}_{1\times p}\mathbf{U}^T$$
$$=\sum_{j=1}^k\frac{\gamma_{ij}}{\sigma_j}\mathbf{u}_j^T.$$
\end{proof}

Lemma \ref{thm4} shows that if $\mathbf{b}_i$ can be written as a linear combination of the $\mathbf{v}_j$, then the vectors $\mathbf{b}_i^T\mathbf{X}^{\dagger}$ can be written as a linear combination of the $\mathbf{u}_i$. Next we give the formal definition of the modularity components.\\

\begin{definition}
Suppose $\mathbf{X}_{p\times n}$ is the data matrix, $\mathbf{b}_i$ is the eigenvector corresponding to the $i$-th largest eigenvalue of $\mathbf{B}$, where 
\begin{equation}\label{eq3}
\mathbf{B}=\mathbf{X}^T\mathbf{X}-\frac{\mathbf{d}\mathbf{d}^T}{2m}. 
\end{equation}
Under the assumptions in Lemma \ref{thm3}, let  
\begin{equation}\mathbf{m}_i^T=\mathbf{b}_i^T\mathbf{X}^\dagger=\sum_{j=1}^k\frac{\gamma_{ij}}{\sigma_j}\mathbf{u}_j^T.
\end{equation}
The $i$-th modularity component is defined to be
\begin{equation}\mathbf{c}_i=\frac{\mathbf{m}_i}{\|\mathbf{m}_i\|_2}.
\end{equation}
\end{definition}

By the two lemmas, it can be seen that as long as the assumptions in Lemma \ref{thm3} are met, the modularity components are well-defined, and the definition of $\mathbf{c}_i$ is based on the linear combination of $\mathbf{b}_i^T\mathbf{X}^\dagger$ in terms of the $\mathbf{u}_i$. In the next section some important properties of the modularity components are established.

\section{Properties of the Modularity Components}
In this section some properties of modularity components will be discussed. It will be seen that the properties of modularity components are similar to the ones of principal components. First we will prove that the modularity components, as long as they are well-defined, are perpendicular to each other. Then we will prove that if we project the uncentered data onto the span of the modularity components, then the projection will be a scalar multiple of the modularity vectors. Finally, we will prove that the `importance' of each modularity component is given by its corresponding eigenvalue of $\mathbf{B}$. The first modularity component has the largest modularity, and the $i$-th modularity component has the largest modularity with the constraint that it is perpendicular to the preceding $i-1$ modularity components.
\begin{theorem}\label{thm5}
With the assumptions in Lemma \ref{thm3}, suppose $\mathbf{X}_{p\times n}$ is the unnormalized data matrix, $\mathbf{A}=\mathbf{X}^T\mathbf{X}$, $\mathbf{B}=\mathbf{A}-\mathbf{d}\mathbf{d}^T/(2m)$. Suppose $\mathbf{b}_i$, $\mathbf{b}_j$ are the eigenvectors of $\mathbf{B}$ corresponding to eigenvalues $\lambda_i$ and $\lambda_j$, $1\le i,j\le k-1$, respectively. Then we have 
\begin{equation}\mathbf{B}=(\mathbf{BX}^\dagger)(\mathbf{BX}^\dagger)^T
\end{equation} and $\mathbf{c}_i\perp\mathbf{c}_j$ for $i\neq j$.
\end{theorem}
\begin{proof}
It is sufficient to prove that $\mathbf{m}_i\perp\mathbf{m}_j$ for $i\neq j$. From $\mathbf{A}=\mathbf{X}^T\mathbf{X}$ we have 
$$\mathbf{d}=\mathbf{Ae}=\mathbf{X}^T\mathbf{Xe},$$
$$2m=\mathbf{d}^T\mathbf{e}=\mathbf{e}^T\mathbf{X}^T\mathbf{Xe},$$
where $\mathbf{e}$ is a column vector with all ones. Therefore,
$$\mathbf{B}=\mathbf{A}-\frac{\mathbf{d}\mathbf{d}^T}{2m}
=\mathbf{X}^T\mathbf{X}-\frac{(\mathbf{X}^T\mathbf{Xe})(\mathbf{X}^T\mathbf{Xe})^T}{\mathbf{e}^T\mathbf{X}^T\mathbf{Xe}}$$
$$=\mathbf{X}^T\mathbf{X}-\frac{\mathbf{X}^T\mathbf{Xe}\mathbf{e}^T\mathbf{X}^T\mathbf{X}}{\mathbf{e}^T\mathbf{X}^T\mathbf{Xe}}.$$
Since $\mathbf{X}^T\mathbf{X}\mathbf{X}^{\dagger}=\mathbf{X}^T$ is always true, we have 
$$\mathbf{BX}^\dagger=\mathbf{X}^T-\frac{\mathbf{X}^T\mathbf{Xe}\mathbf{e}^T\mathbf{X}^T}{\mathbf{e}^T\mathbf{X}^T\mathbf{Xe}}.$$
Consequently,
$$(\mathbf{BX}^\dagger)(\mathbf{BX}^\dagger)^T$$
$$=\bigg(\mathbf{X}^T-\frac{\mathbf{X}^T\mathbf{Xe}\mathbf{e}^T\mathbf{X}^T}{\mathbf{e}^T\mathbf{X}^T\mathbf{Xe}}\bigg)\bigg(\mathbf{X}-\frac{\mathbf{X}\mathbf{ee}^T\mathbf{X}^T\mathbf{X}}{\mathbf{e}^T\mathbf{X}^T\mathbf{Xe}}\bigg)$$
$$=\mathbf{X}^T\mathbf{X}-\frac{2\mathbf{X}^T\mathbf{Xe}\mathbf{e}^T\mathbf{X}^T\mathbf{X}}{\mathbf{e}^T\mathbf{X}^T\mathbf{Xe}}$$
$$+\frac{(\mathbf{e}^T\mathbf{X}^T\mathbf{Xe})\mathbf{X}^T\mathbf{Xe}\mathbf{e}^T\mathbf{X}^T\mathbf{X}}{(\mathbf{e}^T\mathbf{X}^T\mathbf{Xe})^2}$$
$$=\mathbf{X}^T\mathbf{X}-\frac{\mathbf{X}^T\mathbf{Xe}\mathbf{e}^T\mathbf{X}^T\mathbf{X}}{\mathbf{e}^T\mathbf{X}^T\mathbf{Xe}}.$$
Therefore $\mathbf{B}=(\mathbf{BX}^\dagger)(\mathbf{BX}^\dagger)^T$. Since $\mathbf{B}\mathbf{b}_i=\lambda_i\mathbf{b}_i$, $\mathbf{B}\mathbf{b}_j=\lambda_j\mathbf{b}_j$, $\lambda_i\neq0$, $\lambda_j\neq0$, we have 
$$\mathbf{m}_i^T\mathbf{m}_j=(\mathbf{b}_i^T\mathbf{X}^\dagger)(\mathbf{b}_j^T\mathbf{X}^\dagger)^T$$
$$=\bigg(\frac{1}{\lambda_i}\mathbf{b}_i^T\mathbf{BX}^\dagger\bigg)\bigg(\frac{1}{\lambda_j}\mathbf{b}_j^T\mathbf{BX}^\dagger\bigg)^T$$
$$=\frac{1}{\lambda_i\lambda_j}\mathbf{b}_i^T(\mathbf{BX}^\dagger)(\mathbf{BX}^\dagger)^T\mathbf{b}_j=\frac{1}{\lambda_i\lambda_j}\mathbf{b}_i^T\mathbf{B}\mathbf{b}_j$$
$$=\frac{1}{\lambda_i}\mathbf{b}_i^T\mathbf{b}_j=0,$$
so 
$$\mathbf{c}_i^T\mathbf{c}_j=\frac{\mathbf{m}_i^T\mathbf{m}_j}{\|\mathbf{m}_i\|_2\|\mathbf{m}_j\|_2}$$ 
implies $\mathbf{c}_i\perp\mathbf{c}_j$ for $i\neq j$.
\end{proof}

From Theorem \ref{thm5}, it can be seen that the modularity components are orthogonal to each other. Next we prove that the projection of the uncentered data onto the span of $\mathbf{c}_i$ is a scalar multiple of $\mathbf{b}_i$. 

\begin{theorem}\label{thm6}
With the assumptions in Lemma \ref{thm3}, let $\mathbf{P}_{\mathbf{c}_i}$ be the projector onto the span of $\mathbf{c}_i$. Then we have 
\begin{equation}\mathbf{P}_{\mathbf{c}_i}\mathbf{X}=\frac{1}{\|\mathbf{m}_i\|_2}\mathbf{c}_i\mathbf{b}_i^T.
\end{equation} 
\end{theorem}
\begin{proof}
$$\mathbf{P}_{\mathbf{c}_i}\mathbf{X}=\mathbf{c}_i\mathbf{c}_i^T\mathbf{X}=\frac{1}{\|\mathbf{m}_i\|_2}\mathbf{c}_i\mathbf{m}_i^T\mathbf{U\Sigma V}^T$$
$$=\frac{1}{\|\mathbf{m}_i\|_2}\mathbf{c}_i\bigg(\sum_{j=1}^k\frac{\gamma_{ij}}{\sigma_j}\mathbf{u}_j^T\bigg)\mathbf{U\Sigma V}^T$$
$$=\frac{1}{\|\mathbf{m}_i\|_2}\mathbf{c}_i\begin{pmatrix}
\frac{\gamma_{i1}}{\sigma_1} & \frac{\gamma_{i2}}{\sigma_2} & \cdots & \frac{\gamma_{ik}}{\sigma_k} & 0 & \cdots & 0
\end{pmatrix}_{1\times p}\Sigma\mathbf{V}^T$$
$$=\frac{1}{\|\mathbf{m}_i\|_2}\mathbf{c}_i\begin{pmatrix}
\gamma_{i1} & \gamma_{i2} & \cdots &\gamma_{ik} & 0 & \cdots & 0
\end{pmatrix}_{1\times n}\mathbf{V}^T$$
$$=\frac{1}{\|\mathbf{m}_i\|_2}\mathbf{c}_i\sum_{j=1}^k\gamma_{ij}\mathbf{v}_i^T=\frac{1}{\|\mathbf{m}_i\|_2}\mathbf{c}_i\mathbf{b}_i^T.$$
\end{proof}

This property is similar to that of principal components in the sense that if we project the data onto the span of the components, we get a scalar multiple of a vector, and the vector can give the clusters in the data based on the signs of the entries in the eigenvectors. Finally, we can prove that if we look at $\mathbf{X}$ in the space perpendicular to $\mathbf{c}_1$, $\mathbf{c}_2$, $\cdots$, $\mathbf{c}_{i-1}$, then the projection onto the span of $\mathbf{c}_i$ will give us the largest modularity, and the projection is just $\mathbf{b}_i$.
\begin{theorem}\label{thm7}
With the assumptions in Lemma \ref{thm3}, 
\begin{equation}
\beta_i=\frac{1}{\|\mathbf{m}_i\|_2^2},\text{ }1\le i \le k-1.
\end{equation} 
Moreover, let $\mathbf{X}_1=\mathbf{X}$ and for $1<i\le k-1$, let
\begin{equation}\mathbf{X}_i=\mathbf{X}-\sum_{j=1}^{i-1}\mathbf{c}_j\mathbf{c}_j^T\mathbf{X},
\end{equation}
and let $\mathbf{d}_i$, $m_i$ be defined correspondingly. Under these conditions, $\beta_i$ is the largest eigenvalue of $\mathbf{B}_i=\mathbf{X}_i^T\mathbf{X}_i-\mathbf{d}_i\mathbf{d}_i^T/(2m_i)$, and $\mathbf{b}_i$ is the corresponding eigenvector of $\beta_i$.
\end{theorem}
\begin{proof}
For $i=2$, since it is proved in \cite{newman2006modularity} that $\mathbf{b}_1$ is the vector $\bar{\mathbf{s}}$ that maximizes $Q$ in Equation \ref{eq2}, we have
$$Q_{max1}=\mathbf{b}_1^T\mathbf{B}\mathbf{b}_1=\beta_1\mathbf{b}_1^T\mathbf{b}_1=\beta_1.$$
By Theorem \ref{thm5}, 
$$\max_{\|\mathbf{s}\|_2=1}\mathbf{s}^T\mathbf{B}\mathbf{s}=\max_{\|\mathbf{s}\|_2=1}\mathbf{s}^T(\mathbf{BX}^\dagger)(\mathbf{BX}^\dagger)^T\mathbf{s}$$
$$=\max_{\|\mathbf{s}\|_2=1}\|(\mathbf{BX}^\dagger)^T\mathbf{s}\|_2^2=\max_{\|\mathbf{s}\|_2=1}\|(\mathbf{X}^{\dagger})^T\mathbf{B}\mathbf{s}\|_2^2$$
$$=\|(\mathbf{X}^{\dagger})^T\mathbf{B}\mathbf{b}_1\|_2^2=\|(\mathbf{X}^{\dagger})^T\beta_1\mathbf{b}_1\|_2^2=\|\beta_1\mathbf{m}_1\|_2^2=\beta_1.$$
Therefore $\beta_i=1/\|\mathbf{m}_i\|_2^2$. Then $\mathbf{X}_2$ is defined by
$$\mathbf{X}_2=\mathbf{X}-\mathbf{c}_1\mathbf{c}_1^T\mathbf{X}=(\mathbf{I}-\mathbf{c}_1\mathbf{c}_1^T)\mathbf{X}.$$
Since $\mathbf{I}-\mathbf{c}_1\mathbf{c}_1^T$ is idempotent, we have
$$\mathbf{X}_2^T\mathbf{X}_2=\mathbf{X}^T(\mathbf{I}-\mathbf{c}_1\mathbf{c}_1^T)\mathbf{X}=\mathbf{X}^T\mathbf{X}-\mathbf{X}^T\mathbf{c}_1\mathbf{c}_1^T\mathbf{X}.$$
By Theorem \ref{thm6}, we know that 
$\mathbf{c}_1\mathbf{c}_1^T\mathbf{X}=\mathbf{c}_1\mathbf{b}_1^T/\|\mathbf{m}_1\|_2$, so 
$\mathbf{c}_1^T\mathbf{X}=\sqrt{\beta_1}\mathbf{b}_1^T$ and then
$$\mathbf{X}_2^T\mathbf{X}_2=\mathbf{X}^T\mathbf{X}-\beta_1\mathbf{b}_1\mathbf{b}_1^T.$$
Plug $\mathbf{X}_2^T\mathbf{X}_2$ into 
$$\mathbf{B}_2=\mathbf{X}_2^T\mathbf{X}_2-\frac{\mathbf{d}_2\mathbf{d}_2^T}{2m_2}
=\mathbf{X}_2^T\mathbf{X}_2-\frac{\mathbf{X}_2^T\mathbf{X}_2\mathbf{ee}^T\mathbf{X}_2^T\mathbf{X}_2}{\mathbf{e}^T\mathbf{X}_2^T\mathbf{X}_2\mathbf{e}},$$
and notice that $\mathbf{b}_1^T\mathbf{e}=0$ (because $\mathbf{b}_1$ and $\mathbf{e}$ are eigenvectors corresponding to different eigenvalues of $\mathbf{B}$) to produce
$$\mathbf{B}_2=\mathbf{B}-\beta_1\mathbf{b}_1\mathbf{b}_1^T.$$ 
So by Brauer's theorem \cite{meyer2000matrix}(Exercise 7.1.17), the eigenpairs of $\mathbf{B}_2$ are the ones of $\mathbf{B}_1$ with $(\beta_1, \mathbf{b}_1)$ replaced by an eigenpair with zero eigenvalue. So $\beta_2$ is the largest eigenvalue of $\mathbf{B}_2$ and $\mathbf{b}_2$ is the eigenvector of $\mathbf{B}_2$ corresponding to $\beta_2$.\\
For the cases when $2<i\le k-1$, let 
$$Q_{i-1}=\mathbf{s}^T\mathbf{B}_{i-1}\mathbf{s}.$$
Notice that $\mathbf{b}_{i-1}$ is the vector $\mathbf{s}$ that maximizes $Q_{i-1}$. Then by similar steps we can prove that $\|\mathbf{m}_{i-1}\|_2^2=1/\beta_{i-1}$. Then $\mathbf{X}_i$ can be defined by
$$\mathbf{X}_i=\mathbf{X}-\sum_{j=1}^{i-1}\mathbf{c}_j\mathbf{c}_j^T\mathbf{X}=(\mathbf{I}-\sum_{j=1}^{i-1}\mathbf{c}_j\mathbf{c}_j^T)\mathbf{X}.$$
It is easy to see that $\mathbf{I}-\sum_{j=1}^{i-1}\mathbf{c}_j\mathbf{c}_j^T$ is idempotent. Then we have 
$$\mathbf{X}_i^T\mathbf{X}_i=\mathbf{X}^T(\mathbf{I}-\sum_{j=1}^{i-1}\mathbf{c}_j\mathbf{c}_j^T)\mathbf{X}$$
$$=\mathbf{X}^T\mathbf{X}-\mathbf{X}^T(\sum_{j=1}^{i-1}\mathbf{c}_j\mathbf{c}_j^T)\mathbf{X}=\mathbf{X}^T\mathbf{X}-\sum_{j=1}^{i-1}\beta_j\mathbf{b}_j\mathbf{b}_j^T.$$
Plug $\mathbf{X}_i^T\mathbf{X}_i$ into
$$\mathbf{B}_i=\mathbf{X}_i^T\mathbf{X}_i-\frac{\mathbf{d}_i\mathbf{d}_i^T}{2m_i}
=\mathbf{X}_i^T\mathbf{X}_i-\frac{\mathbf{X}_i^T\mathbf{X}_i\mathbf{ee}^T\mathbf{X}_i^T\mathbf{X}_i}{\mathbf{e}^T\mathbf{X}_i^T\mathbf{X}_i\mathbf{e}},$$
and notice that $\mathbf{b}_j^T\mathbf{e}=0$ (because $\mathbf{b}_j$ and $\mathbf{e}$ are eigenvectors corresponding to different eigenvalues of $\mathbf{B}$) to produce
$$\mathbf{B}_i=\mathbf{B}-\sum_{j=1}^{i-1}\beta_j\mathbf{b}_j\mathbf{b}_j^T=\mathbf{B}_{i-1}-\beta_{i-1}\mathbf{b}_{i-1}\mathbf{b}_{i-1}^T.$$
So by Brauer's theorem again, the eigenpairs of $\mathbf{B}_i$ are the ones of $\mathbf{B}_{i-1}$ with $(\beta_{i-1}, \mathbf{b}_{i-1})$ replaced by an eigenpair with zero eigenvalue. So $\beta_i$ is the largest eigenvalue of $\mathbf{B}_i$ and $\mathbf{b}_i$ is the eigenvector of $\mathbf{B}_i$ corresponding to $\beta_i$.
\end{proof}

Theorem \ref{thm7} says when we build the new data matrix $\mathbf{X}_i$ from $\mathbf{X}$, $\mathbf{d}_i$ and $m_i$ change. Also $\mathbf{B}_i$ is different from $\mathbf{B}$, but the eigenpairs of $\mathbf{B}$ are retained by $\mathbf{B}_i$ except for the first $i-1$ pairs. The conclusion is that the first modularity component has the largest modularity of the data $\mathbf{X}$. Each succeeding modularity component has the largest modularity with the constraint that it is orthogonal to all previous modularity components.

\section{Conclusion}
In this paper, the concept of modularity components is defined, and some important properties of modularity components are proven. The concept of modularity components can be used to explain why using more than one eigenvectors of the modularity matrix to do data clustering is reasonable. The combination of modularity clustering and modularity components gives a modularity component analysis that has some nice properties similar to the well known principal component analysis.

\bibliographystyle{siam}
\nocite{*}
\bibliography{mybib1}

\newpage

\appendix
\section*{Appendices}
\addcontentsline{toc}{section}{Appendices}
\renewcommand{\thesubsection}{\Alph{subsection}}

\counterwithin{theorem}{section}

\section{Proof of Lemma \ref{thm3}} \label{app1}
The lemma is based on a theorem from \cite{bunch1978rank} about the interlacing property of a diagonal matrix and its rank-one modification and how to calculate the eigenvectors of a diagonal plus rank one (DPR1) matrix \cite{meyer2000matrix}. The theorem can also be found in \cite{wilkinson1965algebraic}.

\begin{theorem}\label{thm1}
Let $\mathbf{C}=\mathbf{D}+\rho\mathbf{vv}^T$, where $\mathbf{D}$ is diagonal, $\|\mathbf{v}\|_2=1$. Let $d_1\le d_2\le \cdots\le d_n$ be the eigenvalues of $\mathbf{D}$, and let $\tilde{d}_1\le \tilde{d}_2\le \cdots\le \tilde{d}_n$ be the eigenvalues of $\mathbf{C}$. Then $\tilde{d}_1\le d_1\le \tilde{d}_2\le d_2\le\cdots \le\tilde{d}_n\le d_n$ if $\rho<0$. If the $d_i$ are distinct and all the elements of $\mathbf{v}$ are nonzero, then the eigenvalues of $\mathbf{C}$ strictly separate those of $\mathbf{D}$.
\end{theorem}
\begin{corollary}\label{thm2}
With the notations in Theorem \ref{thm1}, the eigenvector of $\mathbf{C}$ corresponding to the eigenvalue $\tilde{d}_i$ is given by $(\mathbf{D}-\tilde{d}_i\mathbf{I})^{-1}\mathbf{v}$.
\end{corollary}

Theorem \ref{thm1} tells us the eigenvalues of a DPR1 matrix are interlaced with the eigenvalues of the original diagonal matrix. Next we will write the eigenvector corresponding to the positive eigenvalues of a modularity matrix as a linear combination of the eigenvectors of the corresponding adjacency matrix.\\
\newline
With the notations in Section 1, since $\mathbf{A}=\mathbf{X}^T\mathbf{X}$, then if the SVD of $\mathbf{X}$ is $\mathbf{X}=\mathbf{U\Sigma V}^T$, then 
$$\mathbf{A}=\mathbf{V}\mathbf{\Sigma^T\Sigma}\mathbf{V}^T=\mathbf{V}\mathbf{\Sigma_\mathbf{A}}\mathbf{V}^T,$$
where $\Sigma_\mathbf{A}$ is an $n\times n$ diagonal matrix. Suppose the rows and columns of $\mathbf{A}$ are ordered such that
$\mathbf{\Sigma_\mathbf{A}}=diag(\alpha_1, \alpha_2, \cdots, \alpha_n)$, where $\alpha_1>\alpha_2>\cdots>\alpha_k>\alpha_{k+1}=\cdots=\alpha_n=0$. Let $\mathbf{V}=\begin{pmatrix}\mathbf{v}_1&\mathbf{v}_2&\cdots&\mathbf{v}_n\end{pmatrix}$. Similarly, since $\mathbf{B}$ is symmetric, it is orthogonally similar to a diagonal matrix. Suppose the eigenvalues of $\mathbf{B}$ are $\beta_1,\beta_2,\cdots,\beta_n$ with largest $k-1$ eigenvalues $\beta_1>\beta_2>\cdots>\beta_{k-1}$.
\begin{proof}
Since $\mathbf{B}=\mathbf{A}-\mathbf{d}\mathbf{d}^T/(2m)$, we have
$$\mathbf{B}=\mathbf{A}-\frac{\mathbf{d}\mathbf{d}^T}{2m}=\mathbf{V}\mathbf{\Sigma_\mathbf{A}}\mathbf{V}^T-\frac{\mathbf{d}\mathbf{d}^T}{2m}=\mathbf{V}(\mathbf{\Sigma_\mathbf{A}}+\rho\mathbf{y}\mathbf{y}^T)\mathbf{V}^T,$$
where $\mathbf{y}=\mathbf{V}^T\mathbf{d}/\|\mathbf{V}^T\mathbf{d}\|_2$ and $\rho=-\|\mathbf{V}^T\mathbf{d}\|_2^2/(2m)$. Since $\mathbf{\Sigma_\mathbf{A}}+\rho\mathbf{y}\mathbf{y}^T$ is also symmetric, it is orthogonally similar to a diagonal matrix. So we have 
$$\mathbf{B}=\mathbf{V}\mathbf{U}^\prime\mathbf{\Sigma_\mathbf{B}}\mathbf{U}^{\prime T}\mathbf{V}^T,$$
where $\mathbf{U}^\prime$ is orthogonal and $\mathbf{\Sigma_\mathbf{B}}$ is diagonal. Since $\mathbf{\Sigma_\mathbf{A}}+\rho\mathbf{y}\mathbf{y}^T$ is a DPR1 matrix, $\rho<0$ and $\|\mathbf{y}\|_2=1$, the interlacing theorem applies to the eigenvalues of $\mathbf{A}$ and $\mathbf{B}$. More specifically, we have
$$\alpha_k<\beta_{k-1}<\alpha_{k-1}<\beta_{k-2}<\cdots<\beta_2<\alpha_2<\beta_1<\alpha_1.$$
The strict inequalities hold because of our assumptions. Let $\mathbf{B}_1=\mathbf{\Sigma_\mathbf{A}}+\rho\mathbf{y}\mathbf{y}^T$. Since $\mathbf{B}=\mathbf{V}\mathbf{B}_1\mathbf{V}^T$, we have $\mathbf{B}\mathbf{V}=\mathbf{V}\mathbf{B}_1$. Suppose $(\lambda,\mathbf{u})$ is an eigenpair of $\mathbf{B}_1$, then
$$\mathbf{B}\mathbf{V}\mathbf{u}=\mathbf{V}\mathbf{B}_1\mathbf{u}=\lambda\mathbf{V}\mathbf{u}$$
implies that $(\lambda,\mathbf{u})$ is an eigenpair of $\mathbf{B}_1$ if and only if $(\lambda,\mathbf{Vu})$ is an eigenpair of $\mathbf{B}$. By Corollary \ref{thm2}, the eigenvector of $\mathbf{B}_1$ corresponding to $\beta_i$, $1\le i \le k-1$ is given by
$$\mathbf{p}_i=(\mathbf{\Sigma_{\mathbf{A}}}-\beta_i\mathbf{I})^{-1}\mathbf{y}=(\mathbf{\Sigma_{\mathbf{A}}}-\beta_i\mathbf{I})^{-1}\frac{\mathbf{V}^T\mathbf{d}}{\|\mathbf{V}^T\mathbf{d}\|_2},$$
and hence the eigenvector of $\mathbf{B}$ corresponding to $\beta_i$, $1\le i \le k-1$ is given by
$$\mathbf{b}_i=\mathbf{V}\mathbf{p}_i=\mathbf{V}(\mathbf{\Sigma_{\mathbf{A}}}-\beta_i\mathbf{I})^{-1}\frac{\mathbf{V}^T\mathbf{d}}{\|\mathbf{V}^T\mathbf{d}\|_2}$$
$$=\frac{1}{\|\mathbf{d}\|_2}\sum_{j=1}^n\frac{\mathbf{v}_j^T\mathbf{d}}{\alpha_j-\beta_i}\mathbf{v}_j.$$
Since $\mathbf{d}=\mathbf{Ae}=\mathbf{V\Sigma_\mathbf{A}}\mathbf{V}^T\mathbf{e}$ where $\mathbf{e}$ is a column vector with all ones, we have 
$$\mathbf{v}_j^T\mathbf{d}=\mathbf{v}_j^T\mathbf{V\Sigma_\mathbf{A}}\mathbf{V}^T\mathbf{e}
=\mathbf{e}_j^T\mathbf{\Sigma_\mathbf{A}}\mathbf{V}^T\mathbf{e}.$$
Since $\text{rank}(\mathbf{A})=k$, we have $\mathbf{v}_j^T\mathbf{d}=0$ for $j>k$. Therefore, the eigenvector of $\mathbf{B}$ corresponding to $\beta_i$, $1\le i \le k-1$ is given by
$$\mathbf{b}_i=\sum_{j=1}^k\gamma_{ij}\mathbf{v}_j,$$
where
$$\gamma_{ij}=\frac{\mathbf{v}_j^T\mathbf{d}}{(\alpha_j-\beta_i)\|\mathbf{d}\|_2}.$$
\end{proof}

\section{Examples satisfy the assumptions in Lemma \ref{thm3}} \label{app2}
We used two subsets of the popular MNIST data set from the literature, and the data set is described below.\\
\newline
The PenDigit data sets are subsets of the widely used MNIST database \cite{lecun1998gradient}\cite{zhang2002large}\cite{hertz2004boosting}\cite{chitta2012efficient}\cite{race2014determining}. The original data contains a training set of 60,000 handwritten digits from 44 writers. The first subset used in the experiments contains some of the digits 1, 5 and 7\footnote[1]{The data can be downloaded at http://www.kaggle.com/c/digit-recognizer/data}. The second subset used contains some of the digits 1, 7 and 9. Each piece of data is a row vector converted from a grey-scale image. Each image is 28 pixels in height and 28 pixels in width, so there are 784 pixels in total. Each row vector contains the label of the digit and the lightness of each pixel. Lightness of a pixel is represented by a number from 0 to 255 inclusively, and smaller numbers represent lighter pixels. \\
\newline
The $\mathbf{X}^T\mathbf{X}$ matrix of the 1-5-7 subset has 644 eigenvalues $\alpha_i$ that are positive, and the largest 643 eigenvalues $\beta_i$ of the $\mathbf{B}$ matrix are different from both $\alpha_i$ and $\alpha_{i+1}$. The $\mathbf{X}^T\mathbf{X}$ matrix of the 1-7-9 subset has 623 eigenvalues $\alpha_i$ that are positive, and the largest 622 eigenvalues $\beta_i$ of the $\mathbf{B}$ matrix are different from $\alpha_i$ and $\alpha_{i+1}$. Thus we conclude that these examples satisfy the assumptions in Lemma \ref{thm3}.

\end{document}